\documentclass[letterpaper, 10 pt, conference]{ieeeconf}
\IEEEoverridecommandlockouts

\usepackage[ansinew]{inputenc}
\usepackage{color,placeins}
\usepackage{hyperref}
\usepackage{graphicx,amsmath}
\usepackage{graphics} 
\usepackage{epsfig,epstopdf} 
\usepackage{mathptmx} 
\usepackage{times} 
\usepackage{amsmath} 
\usepackage{amssymb}  
\usepackage{amsfonts}
\usepackage{amsmath}
\usepackage{amssymb}
\usepackage{psfrag}
\usepackage{float}
\usepackage{cite}

\newtheorem{theorem}{Theorem}[section]
\newtheorem{remark}{Remark}[section]

\newtheorem{assumption}{Assumption}[section]

\def\bi{\begin{itemize}}
\def\ei{\end{itemize}}
\def\bn{\begin{enumerate}}
\def\en{\end{enumerate}}
\def\bq{\begin{eqnarray}}
\def\eq{\end{eqnarray}}
\def\be{\begin{equation}}
\def\ee{\end{equation}}
\def\bea{\begin{eqnarray}}
\def\eea{\end{eqnarray}}
\def\beann{\begin{eqnarray*}}
\def\eeann{\end{eqnarray*}}
\def\bsea{\begin{subeqnarray}}
\def\esea{\end{subeqnarray}}
\def\bmat{\left[ \begin{array}}
\def\emat{\end{array} \right]}
\def\QED{\mbox{\rule[0pt]{1.5ex}{1.5ex}}}

\newfont{\BB}{msbm10}
\newfont{\bb}{msbm8}

\newcommand{\bmx}{\begin{matrix}}
\newcommand{\emx}{\end{matrix}}
\newcommand{\ba}{\begin{array}}
\newcommand{\ea}{\end{array}}

\def\nn{\nonumber}
\def\bq{\begin{eqnarray}}
\def\eq{\end{eqnarray}}

\def\bsmat{\left[ \begin{smallmatrix}}
\def\esmat{\end{smallmatrix} \right]}

\title{\LARGE \bf Distributed Formation Control of Nonholonomic Mobile Robots by Bounded Feedback in the Presence of Obstacles}

\author{Thang Nguyen and Hung M. La
\thanks{Thang Nguyen and Hung M. La are with the Advanced Robotics and Automation (ARA) Lab, Department of Computer Science and Engineering, University of Nevada, Reno, NV89557, USA (e-mail: thangnthn@gmail.com, hla@unr.edu).}
\thanks{This work was partially supported by the National Science Foundation  under grant NSF-NRI \#1426828,  the National Aeronautics and Space Administration (NASA)  under Grant No. NNX15AI02H issued through the Nevada NASA Research Infrastructure Development Seed Grant, and the University of Nevada, Reno.}
}

\begin{document}

\maketitle
\thispagestyle{empty}
\pagestyle{empty}

\begin{abstract}
The problem of distributed formation control of nonholonomic mobile robots is addressed in this paper, in which the robots are designed to track a formation. Collision avoidance among agents is guaranteed using a control law based on a repulsive force. In an uncertain environment where obstacles exist, the construction of a repulsive force and rotational direction enables agents to avoid and pass the obstacles. The control inputs of each robot are designed to be bounded. Numerical simulations with different formations are implemented to demonstrate the efficacy of the proposed scheme.
\end{abstract}


\section{Introduction}
\label{Introduction}
Research on the collective behavior of an autonomous mobile robot system has been extensively conducted for many years. Recent development of flocking control for mobile robots with bounded feedback is designed based on a centralized approach \cite{Han_CASE} and later extended to a decentralized  approach \cite{Nguyen_CCC}.  In many applications, the employment of a single complicated robot system can be replaced by invoking a coordination of a multi-agent system with much simpler configurations, whose advantages can be scalability, flexible deployment, cheaper cost, reliability, etc. As a result, more complex tasks can be achieved by using a group of small mobile robots with lower cost and higher efficiency than a complex unit; see \cite{La_TC_2013,mastellone2008,Nguyen2014, La2012Elsevier,TannerJadbabaiePappas04,LiangLee2006,La_SMCA_2015,nguyen2015formation,cruz2007decentralized} and references therein. In this paper, we address the problem of formation control for a nonholonomic mobile robotic system with guaranteed obstacle/collision avoidance while maintaining bounded physical signals. Our study is motivated by the physical constraints imposed on the motor speed which cannot be as large as desired due to the limited electric current.

Coordination of multi-agent systems has been studied in a variety of scenarios with various control approaches \cite{Nguyen_IEEE_CNS, Saber2006,Egerstedt2001,Fredslund2002,tanner2004leader,La2013JC,Hu2010,Li2013,nguyen2017bounded,pham2017distributed}. \cite{Saber2006} provides a general flocking framework in studying the collective behavior of a multi-agent system. Then, adaptive and optimal flocking controls, respectively, are proposed in \cite{La_IROS09, La_SMC09}, and a decentralized flocking control with a minority of informed agents is studied in \cite{La_ICIEA}. Cooperative learning and active sensing based on flocking control are reported in \cite{La_CYBER, La_CASE13}. In \cite{Egerstedt2001}, a platform-independent approach is proposed to design tracking controllers which  is independent of the coordination scheme. \cite{tanner2004leader} studies the stability properties of mobile  agent  formations  which  are  based  on  leader-following, where how  leader  behavior affects the interconnection errors observed in the formation is investigated using leader to formation stability gains. \cite{Li2013} considers the distributed tracking control problem of multi-agent systems in the leader-follower context, where the input of the leader is bounded and unknown to the followers. The coordination control of the activities of multiple agents in cluttered and noisy environments is addressed in \cite{La_ICRA10,La2013JC, Dang_MFI}. In other work, the multi-agent formation control with reinforcement learning for predator avoidance is reported \cite{La_IEEE_CST}. In \cite{Fredslund2002}, the formation control problem of achieving global behavior in a group of distributed robots  is investigated using only local sensing and minimal communication. \cite{Hu2010} proposes a distributed tracking control scheme with distributed estimators for a leader-follower multi-agent system with measurement noises and directed interconnection topology.

We are interested in the nonholonomic model of a mobile robot which has been studied in several papers: \cite{mastellone2008,cruz2007decentralized,Park2011}. Specifically, the problem of formation control for a system of unicycle-type mobile robots will be addressed. A plethora of papers in the literature are focused on designing controller for achieving various formations without considering obstacle/collision avoidance. Recent developments have witnessed a variety of control schemes dealing with obstacle/collision avoidance in formation control for nonholonomic mobile robots. Unlike the double integrator model, the nonholonomic nature of the unicycle-type mobile robot is challenging in dealing with obstacle/collision avoidance. \cite{Park2011} develops a leader-follower-based adaptive formation control method for electrically driven nonholonomic mobile robots with limited information, where an adaptive observer is developed without the information of the velocity measurement and the formation control part is constructed to obtain the desired formation and achieve the collision avoidance. 
In \cite{cruz2007decentralized}, a scalable multi-vehicle platform is developed to address some cooperative missions for a multi-vehicle control system. The multi-agent system can avoid any collision among agents and escape obstacles, which employs a potential function to construct a control law for each agent. The work in \cite{mastellone2008} proposes a decentralized control scheme for unicycle-type mobile robots to form a formation while avoiding collision/obstacle avoidance. The common feature of these works in constructing a control law for obstacle/collision avoidance is that the potential functions and the control input are unbounded. Here, we aim to design bounded control input with a bounded practical potential function to address the formation control problem for a unicycle-type mobile robotic system.


Bounded control input can be solved for tracking control of a single unicycle robot using the method proposed in \cite{lee2001tracking}. However, the controller is not designed to avoid any obstacles yet. In this paper, we propose a control scheme for a multi-vehicle system based on the approach in \cite{mastellone2008}. In contrast to \cite{mastellone2008}, our control method is able to bound the control inputs. For obstacle/collision avoidance, we employ a practically bounded potential function to exert repulsive forces, which is different from the one in other work \cite{mastellone2008}. Obstacles can be avoided with the help of a virtual agent, which is proposed in \cite{LiangLee2006}. Note that in \cite{LiangLee2006}, constraints on the speeds of the mobile robot are not considered. The work in \cite{Nguyen_CCC} focuses on the control design for the torque level while the control inputs in this paper are the angular and translational speeds. Furthermore, the obstacle avoidance is considered in this paper, which is different from \cite{Nguyen_CCC}. Theoretical analysis is conducted using Lyapunov functions which guarantees obstacle/collision avoidance.


The paper is organized as follows. Section \ref{Formulation} presents the problem formulation. Section \ref{Main} describes the main results of the paper where tracking control, collision and obstacle avoidance is introduced. Section \ref{Example} shows a simulation example to illustrate the presented scheme. Finally, some conclusions are given in Section \ref{Conclusion}.

\emph{Notations:} ${\mathbb R}$ and ${\mathbb R}^+$ are the sets of real numbers and nonnegative real numbers, respectively; for $q = [q_1, \ldots, q_n]^T$, $\nabla_q = [\partial/\partial q_1, \ldots, \partial/\partial q_n]^T$ is the del operator (\cite{RileyHobsonBence2006}); for two vectors $a$ and $b$, $a \cdot b$ is their scalar product; $ (a_1,\ldots, a_n)$ is $[a_1^T,\ldots,a_n^T]^T$; $|\cdot|$ is the absolute value of scalars; and $\|\cdot\|$ is the Euclidean norm of vectors.


\section{Problem Formulation}
\label{Formulation}

Consider a collective system of $N$ identical autonomous mobile robots whose nonholonomic dynamics are given as  \cite{mastellone2008}
\bea
\label{mobiledynamics}
  \dot x_i & =& v_i\cos(\theta_i) \nn\\
 \dot y_i &=& v_i \sin(\theta_i) \nn\\
  \dot \theta_i & =& u_i 
\eea
where  $i=1, . . . , N$, $p_i = [x_i, y_i]^T \in \mathbb{R}^2$, and $\theta_i \in \mathbb{R}$
are respectively the position and the heading angle of the $i$-th robot in the inertial frame $Oxy$; $v_i \in \mathbb{R}$ is the translational speed and  $u_i \in \mathbb{R}$ is the angular speed.

Our formation control problem for (\ref{mobiledynamics}) is to obtain the controls $u_i,v_i$ as bounded functions of the collective state $(p_1, \ldots, p_N$, $\theta_1, \ldots, \theta_N$, $v_1, \ldots, v_N$, $u_1, \ldots, u_N)$ in a distributed fashion such that the following multiple goals are achieved:
\begin{itemize}
  \item[G1)] \emph{Reference tracking:}
  \begin{align}
    \lim_{t\to\infty}(\dot p_i(t) - \dot p_{id}(t)) = 0, \forall i,j = 1, \ldots, N
  \end{align}
where $p_{id}$ is the reference trajectory of agent $i$, which is generated from a formation configuration.
  \item[G2)] \emph{Collision avoidance:} $r_{ij}(t) = \|p_i(t) - p_j(t)\| \ge r_0, \forall t \ge 0, \forall i \ne j$
  \item[G3)] \emph{Obstacle avoidance:} $r_{ij}(t) \le R_0, \forall t \ge 0, \forall i \ne j$.
\end{itemize}

To achieve the goals G2) and G3), we consider the  coordination function
\be
\label{Va_def}
     V_a=\sum\limits_{i=1}^{i=N}\sum\limits_{j=1,j\ne i}^{j=N}V_{ij}=\sum\limits_{i=1}^{i=N}\sum\limits_{j=1,j\ne i}^{j=N}V(r_{ij})
\ee
where $V:\mathbb{R}^+ \to \mathbb{R}^+$ is a function satisfying
\begin{itemize}
  \item[P1)] there are positive constants $V_m$ and an $r \in [a, b]$ such that
  \be
	 0 \le V(r) \le V_m;\nn
   \ee
  \item[P2)] $V(r)$ is decreasing and continuously differentiable on $[a, b]$;
  \item[P3)] $\lim\limits_{r \to a^+}V(r) = V_m$.
\end{itemize}

It is seen that by maintaining $V_a < V_m$, we have $V(r_{ij}(t)) < V_m, \forall t$, implying that $r_{ij}(t) \ge a, \forall t$. This property will be employed to achieve the goals G2) and G3). Let $[x_a,y_a]^T$ be the coordinate of an obstacle or an agent which is avoided for agent $i$. Similar to \cite{mastellone2008}, we define the avoidance and detection as follows:
\bea
\label{Om}	\Omega&=&\Big\{\left[\begin{matrix}x\\y\end{matrix}\right]: \left[\begin{matrix}x\\y\end{matrix}\right] \in \mathbb{R}^2\Big| \Big\|\left[\begin{matrix}x\\y\end{matrix}\right]-\left[\begin{matrix}x_a\\y_a\end{matrix}\right]\Big \|\leq a\Big\}\\
\label{Ga} \Gamma&=&\Big\{\left[\begin{matrix}x\\y\end{matrix}\right]: \left[\begin{matrix}x\\y\end{matrix}\right]\notin \Omega \Big|a<\Big\|\left[\begin{matrix}x\\y\end{matrix}\right]-\left[\begin{matrix}x_a\\y_a\end{matrix}\right]\Big\|\leq b\Big\}.
\eea
\section{Main results}
\label{Main}
The proposed scheme is divided in three parts. First, a control algorithm is designed for each agent to track its reference trajectory. Second, collision avoidance among agents is addressed using a potential force function. Finally, a rotational angle and a virtual agent are introduced for an agent to escape possible collision with an obstacle.

\subsection{Trajectory Tracking Control}
\label{Tracking}
In this section, we will design a control law for each robot to track a given trajectory. The collision and obstacle avoidance scenarios will be addressed separately later.

Assume the reference trajectory is described by $(x_{di}(t),y_{di}(t))^T$ with bounded derivatives. Denote the position errors as $e_{xi}(t)=x_i(t)-x_{di}(t)$ and $e_{yi}(t)=y_i(t)-y_{di}(t)$. The desired orientation of robot $i$ is
\be\label{thetadi}
	\theta_{di}=\textrm{atan}2(-e_{yi}(t),-e_{xi}(t))
\ee
and the orientation error is $e_{\theta i}(t)=\theta_i(t)-\theta_{di}(t)$. We have the following assumption \cite{mastellone2008}.
\begin{assumption}\label{cosetheta_as}
\be
	\cos (e_{\theta i}(t))\ne 0.
\ee
\end{assumption}
\begin{assumption}\label{thetaddot_as}
Define
\be
	\dot{\hat{\theta}}_{di}=\frac{e_{xi}(t)\dot{\hat{e}}_{yi}-e_{yi}(t)\dot{\hat{e}}_{xi}}{D_i^2}
\ee
where
\bea
	D_i&=&\sqrt{e_{xi}^2+e_{yi}^2} \nn \\
	\dot{\hat{e}}_{xi}&=&\frac{e_{xi}(t)-e_{xi}(t-T)}{T} \nn\\
	\dot{\hat{e}}_{yi}&=&\frac{e_{yi}(t)-e_{yi}(t-T)}{T} \nn
\eea
for some small $T>0$. Hence, $\dot{\hat{\theta}}_{di}$ is a sufficiently smooth estimate of 
\be
	\dot{\theta}_{di}=\frac{e_{xi}(t)\dot{e}_{yi}-e_{yi}(t)\dot{e}_{xi}}{D_i^2}.
\ee
\end{assumption}

As pointed out in (\cite{mastellone2008}) that 
\be
	|\dot{\theta}_{di}-\dot{\hat{\theta}}_{di}|\leq \epsilon_{\theta i}\approx O(T)
\ee
for some positive $\epsilon_{\theta i}$. Our objective is to drive each robot to track its reference trajectory with bounded linear and angular velocities. The approach here is similar to the one presented in \cite{mastellone2008}. However, here we impose physical constraints on the control inputs.

\begin{theorem}\label{tracking_thm}
Consider system (\ref{mobiledynamics}) and the reference trajectory described by $(x_{di},y_{di})^T$ satisfying Assumptions \ref{cosetheta_as} and \ref{thetaddot_as}. Then the robot $i$ can track the reference trajectory with bounded error if the following controller is applied
\bea
	u_i&=&-K_{\theta i} e_{\theta i}+\dot{\hat{\theta}}_{di} \label{ui}\\
	v_i&=&K_i \cos(e_{\theta i})\min({D_{max i},D_i}) \label{vi}
\eea
where $K_{\theta i}$, $K_i$, and $D_{max i}$ are positive design parameters.
\end{theorem}
\begin{proof}
Consider the error dynamics
\bea
	\dot{e}_{xi}&=&v_i (\cos(e_{\theta i})\cos (\theta_{di})-\sin(e_{\theta i})\sin (\theta_{di}))-\dot{x}_{di} \nn \\
	\dot{e}_{yi}&=&v_i (\sin(e_{\theta i})\cos (\theta_{di})+\cos(e_{\theta i})\sin (\theta_{di}))-\dot{y}_{di} \nn \\
	\dot{e}_{\theta i}&=&u_i-\dot{\theta}_{di}. \nn 
\eea
If $D_i=0$, then $e_{xi}=e_{yi}=0$. In this case, we will have perfect tracking. Assume $D_i\ne 0$. Using the expressions
\bea
	\cos(\theta_{di})&=&\frac{-e_{xi}}{D_i} \nn\\
	\sin(\theta_{di})&=&\frac{-e_{yi}}{D_i} \nn
\eea
and controllers in (\ref{ui}), (\ref{vi}), we obtain
\bea
\nn	\dot{e}_{xi}&=&K_i \frac{\min({D_{max i},D_i})}{D_i}(-e_{xi}\cos^2(e_{\theta i})\\
&&+e_{yi}\cos(e_{\theta i})\sin(e_{\theta i}))-\dot{x}_{di} \nn \\
\nn	\dot{e}_{yi}&=&K_i \frac{\min({D_{max i},D_i})}{D_i}(-e_{yi}\cos(e_{\theta i})\sin(e_{\theta i}))\\
&&-e_{yi}\cos^2(e_{\theta i}))-\dot{y}_{di} \nn \\
	\dot{e}_{\theta i}&=&-K_{\theta i} e_{\theta i}+\dot{\hat{\theta}}_{di} -\dot{\theta}_{di}. \nn
\eea
Consider the following Lyapunov function
\be
	V_t=\frac{1}{2}(e^2_{xi}+e^2_{yi}+e^2_{\theta i}).
\ee
The derivative of $V_t$ along the trajectories of the error dynamics is 
\bea
\nn	\dot{V}_t&=&\dot{e}_{xi}e_{xi}+\dot{e}_{yi}e_{yi}+\dot{e}_{\theta i}e_{\theta i}\\
\nn 			&\leq&-K_i \frac{\min({D_{max i},D_i})}{D_i}\cos(e_{\theta i})(e^2_{xi}+e^2_{yi})-e_{xi}\dot{x}_{di}\\
\nn&&-e_{yi}\dot{y}_{di}-|e_{\theta i}|(K_{\theta i}| e_{\theta i}|-\epsilon_{\theta i})	\\
&\leq&-\left[\begin{matrix}e_{xi}\\e_{yi}\\e_{\theta i}\end{matrix}\right]^T P \left[\begin{matrix}e_{xi}\\e_{yi}\\e_{\theta i}\end{matrix}\right]- \left[\begin{matrix}e_{xi}\\e_{yi}\\ |e_{\theta i}|\end{matrix}\right]^T\left[\begin{matrix}\dot{x}_{di}\\ \dot{y}_{di}\\-\epsilon_{\theta i}\end{matrix}\right]\\
\nn &\leq&-\left[\begin{matrix}e_{xi}\\e_{yi}\\e_{\theta i}\end{matrix}\right]^T P \left[\begin{matrix}e_{xi}\\e_{yi}\\e_{\theta i}\end{matrix}\right]+ \left\|\left[\begin{matrix}e_{xi}\\e_{yi}\\ e_{\theta i}\end{matrix}\right]\right\|  \left\| \left[\begin{matrix}\dot{x}_{di}\\ \dot{y}_{di}\\ -\epsilon_{\theta i}\end{matrix}\right]\right\| 
\eea
where
\be
\nn P=\left[\begin{matrix}K_a&0&0\\0&K_a&0\\0&0&K_{\theta i}\end{matrix}\right]
\ee
with
\be
K_a=K_i  \frac{\min({D_{max i},D_i})}{D_i}\cos^2(e_{\theta i}).
\ee
Since $\cos(e_{\theta i})\ne 0$ according to Assumption \ref{cosetheta_as}. Hence, $\dot{V}_t<0$ whenever
\be
	\|e_i\|\geq \frac{\|d\|}{\lambda_{\min}(P)},
\ee
where $e_i=[e_{xi},e_{yi},e_{\theta i}]^T$ and $d_i=[\dot{x}_{di}, \dot{y}_{di}, \epsilon_{\theta i}]^T$.
\end{proof}
\begin{remark}
The control signal in (\ref{vi}) is bounded by an appropriate value of $D_{maxi}$. Since $\dot{\theta}_{di}$ depends on $v_i$ from its definition, and $\theta_i$ and $\theta_{di}$ are bounded, the control input in (\ref{ui}) is also bounded. 
\end{remark}

\subsection{Collision avoidance}\label{Collision}
In this section, we propose a control law for each agent when other agents are in its collision avoidance region. Our strategy is based on a potential function, which exerts a repulsive force for each agent. Here we use the potential function proposed in \cite{LiangLee2006}, which is described as
\be\label{Vij}
	V_{ij}(p_i,p_j) = K_{ij} \ln(\cosh(p_{ij}))\,h_{ij}(p_i,p_j)
\ee
where $K_{ij}>0$, $p_{ij}=\|p_i-p_j\|-c_i$, and 
\be
	h_{ij}(p_i,p_j)=\begin{cases} 1& \text{for } \|p_i-p_j\|<b_i\\0&\text{otherwise}\end{cases}
\ee
with $0<a_i<b_i<c_i$. Here, $a_i$ and $b_i$ are the parameters $a$, $b$ in (\ref{Om}), (\ref{Ga}) respectively.
Let $p_a$ be the coordinates of the object to be avoided for agent $i$. 

Let $\mathcal{N}_i(t)$ denotes the neighbor set of robot $i$ for $i=1, \dots, N$. Similarly to \cite{LiangLee2006}, we choose total structural potential energy from all neighbors around agent $i$ as 
\be
	U_i(p_i)=\sum\limits_{j\in \mathcal{N}_i(t)} V_{ij}(p_i,p_j).
\ee
The interactive structural force between agents $i$ and $j$ is the gradient of the potential energy
\be
	f_{ij}(p_i,p_j)=\Delta V_{ij}=K_{ij}\tanh(p_{ij})\frac{p_j-p_i}{\|p_j-p_i\|}.
\ee
Unlike the counterpart in \cite{LiangLee2006}, this function exhibits a repulsive force when the step function $h_{ij}$ is active. If other agents approach agent $i$ in its avoidance region, the total structural force acting on it is
\be
	F_i(q_i)=\sum\limits_{j\in \mathcal{N}_i(t)} f_{ij}(p_i,p_j).
\ee
Let $f_{ix}$ and $f_{iy}$ be the components of the synthesized force in the x and y coordinates respectively. Define
\be\label{thetadi_a}
	\theta_{di}=\textrm{atan}2(-f_{iy}(t),-f_{ix}(t))
\ee
and
\be\label{Dia}
	D_i=\sqrt{f_{ix}^2+f_{iy}^2}.
\ee
We have the following result.

\begin{theorem}\label{avoidance_thm}
Let 
\be
	V_m= K_{ij} \ln(\cosh(a_i-c_i)).
\ee
Assume that all agents share the same values of $K_{ij}$, $a_i$, $b_i$, $c_i$. Let
\be
	V_a=\sum\limits_{i=1}^{i=N}\sum\limits_{j\in \mathcal{N}_i(t)}V_{ij}.
\ee
If the collective system (\ref{mobiledynamics}) is initiated such as 
\be
	V_m>V_a
\ee
 and the control laws are chosen as
\bea
	u_i&=&-K_{\theta i} e_{\theta i} \label{uia}\\
	v_i&=&K_i \cos(e_{\theta i})\min({D_{max i},D_i}) \label{via}
\eea
where $K_{\theta i}$, $K_i$, and $D_{max i}$ are positive design parameters, then the collision avoidance is guaranteed.
\end{theorem}

\begin{proof}
The derivative of $V_a$ along the trajectories of the robots is
\be
	\dot{V}_a=-\sum\limits_{i=1}^{i=N} K_i \min({D_{max i},D_i})D_i\cos^2(e_{\theta i}),
\ee
which is not positive. Hence, $V_a$ is a nonincreasing function with respect to $t$. Hence,
\be
	V_m>V_a(0)\geq V_a(t)\geq V_{ij}(t).
\ee
From the definition of $V_{ij}$, the above inequality implies that $\|p_i-p_j\|>a_i$ for all $t\geq 0$. This guaranteed that any agents $i$ and $j$ will never collide with each other. Thus, there is no collision among agents when using the control law (\ref{uia}) and (\ref{via}).
\end{proof}

\subsection{Obstacle avoidance}\label{Obstacle}
It is necessary for agents to avoid obstacles in unknown environments while traveling cooperatively with others. An obstacle can have a variety of shapes and sizes. It can be convex or nonconvex. In this paper, we propose a scheme for obstacle avoidance for the agents in the collective system. The description of the obstacle is similar to the one in \cite{LiangLee2006}. 
Let $O_k(x_k,y_k)$ be the center of the obstacle and $r_k$ be its radius. The projection point of agent $i$ onto the surface of the obstacle is denoted as $p_{obs}=[x_{obs},y_{obs}]^T$. As pointed out in \cite{LiangLee2006},
\be
	p_{obs}=\frac{r_{obs}}{\|p_i-O_k\|}p_i+(1-\frac{r_{obs}}{\|p_i-O_k\|})O_k.
\ee
This projection point can be regarded as a virtual robot as it also has position and velocity where its velocity is calculated in \cite{LiangLee2006}. The interaction between agent $i$ and its virtual robot can be described in terms of the potential function in Section \ref{Collision}. Hence, we employ similar control laws as described in Section \ref{Collision} to navigate agent $i$ to avoid any possible collision with the obstacle. When in a safe distance, robot $i$ still needs to navigate to pass the obstacle. Hence, the orientation angle is chosen as the tangential direction of the boundary of the obstacle. This angle can be measured in real time regardless of the shape and size of the obstacle. So, this scheme can help the group to pass obstacles with complex shapes. The reference angle is chosen as
\be
	\theta_{di}=-\frac{\pi}{2}+(\gamma+\beta), \quad \gamma>0,
\ee 
or
\be
	\theta_{di}=\frac{\pi}{2}+(\gamma+\beta), \quad \gamma\leq 0,
\ee 
where $\gamma$ is the angle measured from the heading angle of the robot to the straight line which connects the robot to the obstacle and $\beta$ is the angle which is made from the vector connecting the robot to its reference destination and the $x$ axis. Once agent $i$ escapes the obstacle, the tracking control laws are active again to track its reference trajectory. The control laws are the same to all the agents in the collective system. Hence, the group is able to pass obstacles in uncertain environments.
\begin{figure}[t]
    \centering
    \includegraphics[width=\columnwidth,height=2.3in,keepaspectratio]{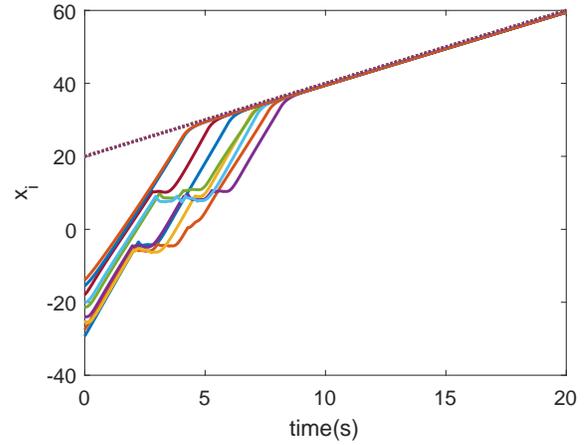}
    \centering
    \caption{Evolution of the x coordinates of agents.}
    \label{x}
\end{figure}

\begin{figure}[t]
    \centering
    \includegraphics[width=\columnwidth,height=2.3in,keepaspectratio]{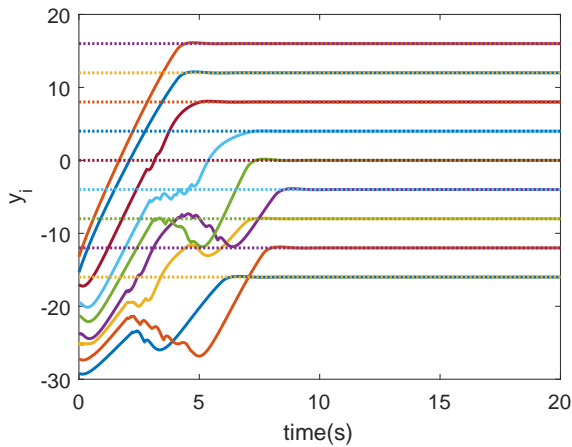}
    \centering
    \caption{Evolution of the y coordinates of agents.}
    \label{y}
\end{figure}

\begin{remark}
Similarly to \cite{mastellone2008}, we use an interactive potential function to avoid collisions among agents and obstacles. In this paper, the control laws (\ref{ui}) and (\ref{vi}) are employed when the collision region between agent $i$ and other agents or obstacle is not active. In contrast, when the collision region is detected, the control laws (\ref{uia}) and (\ref{via}) are in use. This is different from \cite{mastellone2008}, from which we use a similar approach for design and analysis of our control laws. Note that all proposed control laws in this paper are bounded.
\end{remark}

\section{Numerical example}\label{Example}
In this section, we run simulation for a multi-agent system of 9 mobile robots of the model (\ref{mobiledynamics}). The parameters for the potential function $V_{ij}$ in (\ref{Vij}) are $K_{ij}=3$, $a_i=1$, $b_i=2$, $c_i=4$. There are two obstacles whose parameters are $x_{obs}=[0;15]$, $y_{obs}=[-20;-5]$, and $r_{obs}=[3;4]$. The parameters of the control laws in (\ref{ui}) and (\ref{vi}) are
$K_{\theta i}=3$, $K_i=4$, and $D_{maxi}=3$.

\begin{figure}[t]
    \centering
    \includegraphics[width=\columnwidth,height=2.3in,keepaspectratio]{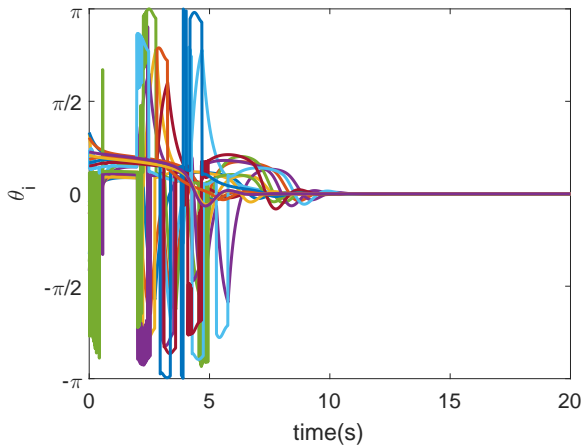}
    \centering
    \caption{Evolution of the heading angles.}
    \label{theta}
\end{figure}

\begin{figure}[t]
    \centering
    \includegraphics[width=\columnwidth,height=2.3in,keepaspectratio]{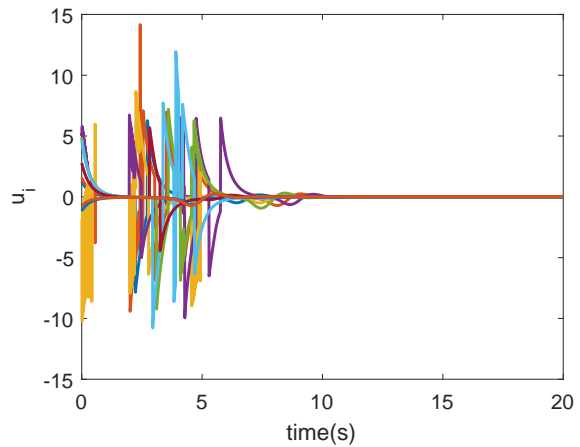}
    \centering
    \caption{Evolution of the angular speeds.}
    \label{u}
\end{figure}

\begin{figure}[t]
    \centering
    \includegraphics[width=\columnwidth,height=2.3in,keepaspectratio]{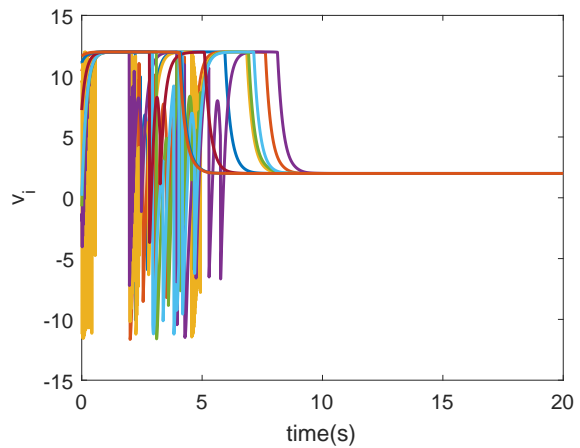}
    \centering
    \caption{Evolution of the linear speeds.}
    \label{v}
    \vspace{-10pt}
\end{figure}

\begin{figure}[t]
    \centering
    \includegraphics[width=\columnwidth,height=2.3in,keepaspectratio]{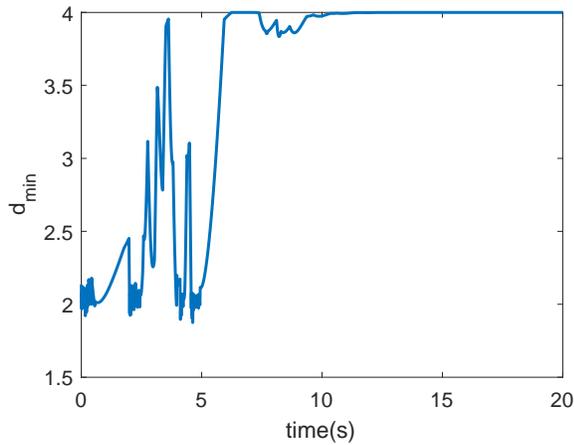}
    \centering
    \caption{Evolution of the minimum distance among agents.}
    \label{dmin}
    \vspace{-10pt}
\end{figure}

\begin{figure}[t]
\vspace{-5pt}
    \centering
    \includegraphics[width=\columnwidth,keepaspectratio]{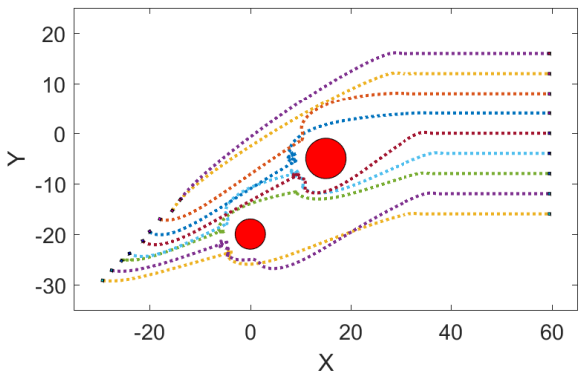}
    \centering
     \vspace{-20pt}
    \caption{Evolution of the agents.}
    \label{traj}

\end{figure}

The formation in this simulation is a line shape, in which the reference trajectory of agent $i$ is
\bea
\nn x_{di}&=&2\,t+20\\
\nn y_{di}&=&4\,i-20.
\eea

The other formations can be done similarly. 

Fig. \ref{x} exhibits the evolution of the x coordinates of the agents. Fig. \ref{y} shows the evolution of the y coordinates of the agents. It is seen that after 10 seconds, all agents converge to their reference trajectories. Note that during $t\in [3,5]s$ some agents avoid the two obstacles and pass them successfully. Fig. \ref{theta} presents the heading angles of the agents, which converge to the same value. The angular and transitional speeds are depicted in Figs. \ref{u} and \ref{v}, which demonstrate that all control inputs are bounded. The minimum distance among agents is depicted in Fig. \ref{dmin}, which shows that there is no collision among them. Finally, the evolution of the agents in the plane is illustrated in Fig. \ref{traj}, where it is clearly seen that all the agents avoid collision with the two obstacles and escape them successfully.


\section{CONCLUSION}
\label{Conclusion}
This paper presents a control scheme for the problem of formation control of a nonholonomic mobile robot system. The proposed control inputs are practically bounded. Collision and obstacle avoidance is guaranteed with the help of bounded potential functions and rotational angles. Numerical simulation has been shown to demonstrate the effectiveness of the proposed approach.

\bibliographystyle{IEEEtran}
\bibliography{IEEEabrv,references}

\end{document}